\documentclass[runningheads]{llncs}
\usepackage{graphicx}
\usepackage{lineno,hyperref}
\usepackage{amsmath}
\usepackage{multirow}
\usepackage{numcompress}
\usepackage{booktabs}
\usepackage{diagbox}
\usepackage{float}

\modulolinenumbers[5]

\begin{document}
\title{Exponential Negation of a Probability Distribution}
%
%
\author{Qinyuan Wu\inst{1} \and
Yong Deng\inst{1,2,3,*}\and
Neal Xiong\inst{4}}
\authorrunning{Qinyuan Wu, Yong Deng}
%
\institute{* Institute of Fundamental and Frontier Science, University of Electronic Science and Technology of China, Chengdu, 610054, China \and
School of Education, Shannxi Normal University, Xi'an, 710062, China \and
School of Knowledge Science, Japan Advanced Institute of Science and Technology, Nomi, Ishikawa 923-1211, Japan\and
Department of Mathematics and Computer Science,Northeastern State University,611 N. Grand Ave, \#323, Webb center, Tahlequah, OK 74464, America
\email{*dengentropy@uestc.edu.cn}\\
}

\maketitle              
\begin{abstract}
  Negation operation is important in intelligent information processing.
  Different with existing arithmetic negation, an exponential negation is presented in this paper. The new negation can be seen as a kind of geometry negation.
  Some basic properties of the proposed negation are investigated, we find that the fix point is the uniform probability distribution.
  The proposed exponential negation is an entropy increase operation and  all the probability distributions will converge to the uniform distribution after multiple negation iterations. The convergence speed of the proposed negation is also faster than the existed negation.
  The number of iterations of convergence is inversely proportional to the number of elements in the distribution.  Some numerical examples are used to illustrate the efficiency of the proposed negation.

\end{abstract}
\section{Introduction}
Knowledge representation and uncertainty measure are important issues in artificial intelligence \cite{kanal2014uncertainty,fu2020comparison,xu2018belief,Fei2019Evidence}.
Probability distribution is an efficient way to describe uncertainty, which has achieved remarkable results in artificial intelligence \cite{solomonoff1986application}, quantum science \cite{pitowsky1989quantum} and many other scientific fields \cite{huelsenbeck2001bayesian} \cite{feller2008introduction}.
However, sometimes it is very difficult to analyze the knowledge directly by the probability distribution itself. A possible way is to analyze the probability distribution t from the negative perspective.
Negative representation is an efficient way to present and analyze information from the opposite side \cite{anjaria2020negation}.
For example, in the process of proving mathematical theorems, contradiction is usually an effective method.
Through this example, we can find that negation can show us the other side of knowledge and help us indirectly acquire the knowledge we need.
\par

The negation operation has been heavily studied.
In previous researches, Zadeh first proposed the negation of probability distribution in his BISC(Berkeley Initiative in Soft Computing) blog.
After that, Yager proposed an important method of negation which has the maximum entropy allocation \cite{yager2014maximum}.
The basic framework of Yager's negation is subtracting a probability by 1.
To some extent, Yager got the expression of negation based on subtraction, which can be seen as a kind of arithmetic negation.
Yager also indicate some basic properties of this negation which gets some superior properties that accord with our intuitive understanding of the laws of nature.
Inspired by Yager's negation, the negation of joint and marginal probability distributions in the binary case is proposed \cite{srivastava2018some}.
To study more about the nature of negative distribution, an uncertainty metrics is developed to measure the uncertainty associated with negative probability distributions \cite{srivastava2019uncertainty}.
A new method is also proposed based on Tsallis entropy \cite{zhang2020extension} which will degenerate to Yager's method in some special cases.
In order to verify the new entropy with good performance in uncertainty measurement, a new Pythagoras fuzzy number inversion method is proposed and the new method can be applied in a service supplier selection system \cite{mao2020negation}.
Combining the negative distribution of probability with evidence theory is also a popular application direction.
According to this idea, a method for determining the weight of the sum of probability distribution based on MCDM is proposed recently \cite{sun2020determining}.
Deng and Jiang proposed a negation transformation for belief structure which is based on maximum uncertainty allocation \cite{deng2020negation}.
The negation of probability distribution is also applied to target recognition based on sensor fusion \cite{gao2019generalization}.
By introducing Dempster-Shafer theory, the negation of probability is extended to basically change the allocation.
Some methods for the negation of BPA are proposed to be seen as a new perspective of uncertainty measure \cite{yin2018negation}.
Some matrix methods for the negation of BPA have also been proposed \cite{luo2019matrix}.
Based on the negation of BPA, a method for conflict management of evidence theory is proposed.
Due to the importance, some other negation of different uncertainty \cite{Xiao2020maximum}, such as Z numbers, are developed \cite{liu2020negation}.

\par
However, how to reasonably negate the probability distribution is still an open issue. All the existing negation methods have the similar structure with Yager's negation, whose aim is to achieve the maximum entropy in the final status. However, one shortcoming is that the maximum entropy can not be obtained when only two elements. For example, a probability as $P(A)=0.7$, $P(B)=0.3$.
All existing methods will obtain the result either $P(A)=0.7$, $P(B)=0.3$ or $P(A)=0.3$, $P(B)=0.7$. None of existing methods can achieve the maximum entropy. To address this issue, an exponential negation is proposed in this article. We solve this problem based on exponent function.
In short, Yager's negation can be seen as arithmetic negation and the proposed method can be seen as geometry negation to some degree.
Although the form of two transformations is different, the transformation we proposed have the same properties as the transformation of Yager in many fundamental ways.
Specially, the two negations are completely different when we have a negation transformation on binary probability distribution.
And the result of exponential negation can better meet our intuitive understanding.
\par
The paper is structured as follows: Section~\ref{one} introduces Yager's negation.
The exponential negation is presented in Section~\ref{two}, we also investigate some basic properties of the negation.
In Section~\ref{three}, we give some numerical examples to support the illustrated negation.
We get the conclusion in Section~\ref{five}.

\section{PRELIMINARIES}
Knowledge representation and uncertainty measure are important issues in artificial intelligence \cite{Fujita2020Heuristic,liu2017change,Deng2020ScienceChina,zhou2019assignment,meng2020uncertainty}.
A lot of methods have been presented, including intuitionistic fuzzy sets \cite{garg2019linguistic}, Z-number \cite{Jiang2019Znetwork}, evidence theory \cite{song2020selfadaptive}, D numbers \cite{liu2020anextended}, evidential reasoning \cite{zhou2017evidential,liao2020deng}, complex evidence theory \cite{Xiao2019complexmassfunction,Xiao2020CED,Xiao2020CEQD},
which are applied in classification \cite{liu2018classifier,xu2017data,li2020robust}, information fusion \cite{Xiao2020GIQ,lai2020multi}, medical diagnosis \cite{cao2019multi,wang2017rumor}, fault diagnosis \cite{Xiao2020Novel},  intrusion detection \cite{mi2018reliability}, reliability analysis \cite{mi2018reliability}, risk analysis and assessment   \cite{pan2020multi,zhang2015multimodel,zhanglimao2019,8926527,Zhou2020RIskAssessment}, and decision making \cite{tang2020dynamic,Feng2020Enhancing,fu2020multiple,fei2020mcdmPFS}.

\label{one}
\subsection{Yager's Negation}
A maximum entropy negation of a probability distribution is proposed by Yager \cite{yager2014maximum}.
Suppose the frame of reference is the set $X=\{x_1,x_2,\dots,x_n\}$, $P=\{p_1,p_2,\dots,p_n\}$ is a probability distribution on $X$.
Yager's negation can be expressed as follows \cite{yager2014maximum}:
\begin{equation}
  \begin{split}
    \overline{p_i}=\frac{1-p_i}{n-1}
    \end{split}
  \end{equation}
  \par
Obviously this negation satisfies some basic properties of probability:
\begin{equation}
  \begin{split}
    \sum_{i}\overline{p_i}=1
    \end{split}
  \end{equation}
\begin{equation}
    \begin{split}
    \overline{p_i}\in[0,1]
    \end{split}
  \end{equation}
\par
For many other inversion operations like real number's negation or the negation of a matrix, the final result which has been negative twice is equal to the initial one.
However, Yager's negation is not involutionary.
\begin{equation}
  \begin{split}
      \overline{\overline{p_i}}\ne p_i
    \end{split}
  \end{equation}
  \par
In the further research, Yager indicated the reason of this unusual property. Yager's negation operation will increase the entropy of a system, and then he raised a new method to measure the entropy of a distribution. The entropy of the distribution and the negative distribution is measured following this two formulas:
\begin{equation}
  \begin{split}
      H(P)=\sum_{i}(1-p_i)(p_i)=1-\sum_{i}p_i^2
    \end{split}
  \end{equation}
\begin{equation}
    \begin{split}
        H(\overline{P})=\sum_{i}(1-\overline{p_i})(\overline{p_i})=1-\sum_{i}\overline{p_i}^2
      \end{split}
    \end{equation}
    \par
Then the entropy increasing in this negation process is shown as follows:
\begin{equation}
  \begin{split}
      H(\overline{P})-H(P)&=\sum_{i}p_i^2-\sum_{i}\overline{p_i}^2\\
      &=\sum _{i}p_i^2-\frac{1}{(n-1)^2}\sum_{i}(1-2p_i+p_i^2)\\
      &=\frac{(n-2)}{(n-1)^2}(n\sum_{i}p_i^2-1)
    \end{split}
    \label{H_yager}
  \end{equation}
\par
As can be seen in Eq.\eqref{H_yager}, it is positive and the negation is an entropy increasing operation.

\section{EXPONENTIAL NEGATION}
\label{two}
In this section, a new method of negation of a probability distribution is proposed and some properties of it will be discussed.
\par
Suppose a set of random variables, $X=\{x_1,x_2,\dots,x_n\}$. It is a mathematical abstract representation of event $A=\{A_1,A_2,\dots,A_n\}$. And let $P=\{p_1,p_2,\dots,p_n\}$ as a probability distribution on $X$.
 $p_i$ represents the probability of occurrence of event $A_i$ and each $p_i$ corresponds to an event.
For a probability distribution, two conditions are satisfied as follows :
\begin{equation}
  \begin{split}
       \sum{p_i}=1
  \end{split}
  \end{equation}

  \begin{equation}
  \begin{split}
       p_i\in[0,1]
  \end{split}
  \end{equation}
\par
Then, the distribution $\overline{P}=\{\overline{p_1},\overline{p_2},\dots,\overline{p_n}\}$ is used to describe the negation of distribution $P$.
$\overline{p_i}$ indicates the information or knowledge of 'not $A_i$'. And each negation probability $\overline{p_i}$ has a corresponding 'not $A_i$'.
\par
\begin{definition}
    Given a probability distribution $P=\{p_1,p_2,...,p_n\}$ on $X=\{x_1,x_2,...,x_n\}$, the corresponding exponential negation of the probability distribution is defined as follows:
\begin{equation}
  \begin{split}
      \overline{p_i}&=\frac{e^{-p_i}}{\sum_i{e^{-p_i}}}\\
      &=(\sum_{i=1}^{n}e^{-p_i})^{-1}e^{-p_i}\\
      &=Ae^{-p_i}
    \end{split}
  \end{equation}
  \label{negation_definition}
\end{definition}
Where e is the natural base.
\par
$A$ is a normal number, shown as follows:
\[
A=(\sum_{i=1}^{n}e^{-p_i})^{-1}
\]
\par
$\overline{P}$ is still a probability distribution and satisfied as following conditions.
\begin{equation}
  \overline{p_i}\in[0,1]
\end{equation}
\begin{equation}
    \sum_{i}\overline{p_i}=1
\end{equation}
\par
Some examples are shown as follows.
\begin{example}
Give a probability distribution as follows:
    \begin{equation}\nonumber
        \begin{split}
          {P}:{p_1}=&1     \\
          {p_i}=&0, (i\ne1,i=2,3,...,n)     \\
        \end{split}
      \end{equation}
    Its corresponding negation, according to Iq \ref{negation_definition}, is shown as follows.\\
        \begin{equation}\nonumber
        \begin{split}
          \overline{P}:\overline{p_1}=&(e^{-1}+n-1)^{-1}e^{-1}     \\
          \overline{p_i}=&(e^{-1}+n-1)^{-1}, (i\ne1,i=2,3,...,n) \\
        \end{split}
      \end{equation}
      Considering $n=2$, we can get the result as:
      \begin{equation}\nonumber
        \begin{split}
          \overline{P}:\overline{p_1}=&(e^{-1}+1)^{-1}e^{-1}=0.2689   \nonumber  \\
          \overline{p_i}=&(e^{-1}+1)^{-1}=0.7311,(i\ne1,i=2,3,...,n)   \nonumber \\
        \end{split}
      \end{equation}
\end{example}
  Specially, when $n=2$, $p_1=1$, $p_2=0$, this distribution is a very special case.
  This binary distribution means a certain system which gets the minimum entropy and maximum information.
  Performing an iterative negation operation on this system will increase its entropy and converge to a uniform distribution.
  This is a reasonable result following the laws of nature, means the maximum entropy.
\begin{example}
Give a probability distribution as follows:
    \label{example2}
    \begin{equation}\nonumber
        \begin{split}
          {P}:{p_1}=&p_2=\frac{1}{2}     \\
          {p_i}=&0,(i\ne1,2,i=3,...,n)     \\
        \end{split}
      \end{equation}
     Its corresponding negation, according to Iq \ref{negation_definition}, is shown as follows.\\
    \begin{equation}\nonumber
    \begin{split}
          \overline{P}:\overline{p_1}=&p_2=(2e^{-1}+n-2)^{-1}e^{-\frac{1}{2}}     \\
          \overline{p_i}=&(2e^{-1}+n-2)^{-1},(i\ne1,i=2,3,...,n) \\
        \end{split}
      \end{equation}
      Considering $n=2$,we can get the result as:
      \begin{equation}\nonumber
        \begin{split}
          {P}:{p_1}&=p_2=\frac{1}{2}    \\
          \overline{P}:\overline{p_1}&=p_2=(2e^{-1})^{-1}e^{-1/2}=\frac{1}{2}
        \end{split}
      \end{equation}
\end{example}
  The negation of $P$ is also $P$.
  From this result, we can guess that the uniform distribution is a fix point of the negation.
\par
Furthermore, the proposed exponential negation also has the property of order reversal. If $p_i>p_j$, we can find that $\overline{p_i}<\overline{p_j}$ from the definition.
It is obvious to get such a result. If the probability of event $A_1$ is greater than event $A_2$, it is reasonable that the probability of 'not $A_1$' is smaller than 'not $A_2$' from experience.
\par
Also, we note that this negation is not involutionary, coincide with Yager's negation  \cite{yager2014maximum} and Heyting intuitionistic logic \cite{heyting1966intuitionism}.
\setcounter{equation}{12}
\begin{equation}
  \begin{split}
    \overline{\overline{p_i}}&=A_{1}e^{-\overline{p_i}}=(\sum_{i}e^{-\overline{p_i}})^{-1}e^{-\overline{p_i}}\\
&=(\sum_{i}e^{-\frac{e^{-p_i}}{\sum_{i}e^{-p_i}}})^{-1}e^{-\frac{e^{-p_i}}{\sum_{i}e^{-p_i}}}\ne p_i
  \end{split}
\end{equation}
\par
This property can be understood by the Example \ref{example2}, the distribution will converge to a uniform distribution.
If the operation is involutionary, we will never get the uniform distribution in this example.
\par
Then, we figure out that uniform probability distribution is a fix point of the exponential negation, and we will use Shannon entropy \cite{shannon1948mathematical} to explain it.
Also, we will give some numerical examples to show the details and help understand.
Considering a uniform probability distribution: $P=\{p_i|p_i=\frac{1}{n}, (i=1,2,...,n)\}$.
By definition of the proposed exponential negation, we can calculate the negation distribution as:
\begin{equation}
  \begin{split}
    \overline{p_i}=\frac{e^{-p_i}}{\sum_{i}e^{-p_i}}=\frac{e^{\frac{1}{n}}}{ne^{\frac{1}{n}}}=\frac{1}{n}
  \end{split}
\end{equation}
\par
So, we can get the negative distribution as: $\overline{P}=\{\overline{p_i}| \overline{p_i}=\frac{1}{n}, (i=1,2,...,n)\}$
From the first perspective, we can easily note that when $p_1=p_2=...=p_i=p_j=...=p_n$, there is no doubt that $\overline{p_1}=\overline{p_2}=...=\overline{p_i}=\overline{p_j}=...=\overline{p_n}$ by the definition.
And following the basic character, $\sum_{i}\overline{p_i}=1$, so the negation of uniform distribution is also the same uniform distribution.
\par
From the second perspective, we explain it by Shannon entropy \cite{shannon1948mathematical}. Shannon's entropy is defined as follows:
 \begin{equation}
  \begin{split}
    H(P)=\sum_{i}p_{i}ln(p_i)
  \end{split}
\end{equation}
\par
$H(P)$ is a significant measurement for information of a probability distribution.
For a probability distribution we obtained, the greater entropy means the less information. Shannon entropy can also be seen as information volume \cite{Deng2020InformationVolume,deng2021fuzzymembershipfunction}
It is easy to prove that a uniform probability distribution has the maximum entropy in all distributions.
\begin{theorem}
Applying the proposed exponential negation to the probability distribution causes it to converge to the maximum entropy state.
\end{theorem}
\begin{proof}
\begin{equation}
  \begin{split}
    H(P)&=-\sum_{i}p_{i}ln(p_i)\\
    H(\overline{P})&=-\sum_{i}Ae^{-p_i}ln(Ae^{-p_i})\\
    H(\overline{P})-H(P)&=-\sum_{i}Ae^{-p_i}ln(Ae^{-p_i})+\sum_{i}p_{i}ln(p_i)\\
    &=\sum_{i}[-Ae^{-p_i}ln(Ae^{-p_i})+p_{i}ln(p_i)]\\
    &=\sum_{i}[(\sum_{i}e^{-p_i})^{-1}e^{-p_i}(ln\sum_{i}e^{-p_i}+p_{i})+p_{i}lnp_i]\\
    &\ge0
  \end{split}
  \label{H}
\end{equation}
\end{proof}
\par
The result which is always positive or equaling to zero shows that every negative distribution will never have smaller entropy than the original one.
When $p_i=\frac{1}{n}$, $H(\overline{P})-H(P)=0$, the uniform distribution has the maximum entropy already, the entropy can't increase anymore. So, after the negation operation, it will stay at the maximum and become the fix point of our negation.
\par
This is also a strong evidence to show that our negation could reflect some essences of the real world.
The second law of thermodynamics \cite{callen1998thermodynamics} claims a theory that for an isolated system, it always changes in the direction of entropy increase.
Our negation operation is consistent with The second law of thermodynamics.
According to above reasoning, we can easily infer that after several negation iterations of entropy increase,
the distribution will approach the fix point. In another perspective, all the probability distribution will converge to uniform distribution.
\section{NUMERAL EXAMPLES}
\label{three}
In this section, three numeral examples are shown and compared with Yager's method. In short, the proposed exponential negation has the following desirable properties.
\par
In all tables, $P_i$ means the distribution, $n$ means the number of negation iteration.
Suppose three representative distribution for numerical calculation.
\begin{itemize}
  \item[1.] Any negation iteration before the final status increase entropy.
  \item[2.] The final status will achieve the fix point, which means all the probability distribution will converge to the uniform distribution.
 \end{itemize}
 \begin{example}
 Give a probability distribution as $P_1=\{p_{11},p_{12}\}=\{0,1\}$
 Table \ref{negation_table1} and Table \ref{negation_table2} show the results after given some iterative exponential negations on the distribution. Figure \ref{ExponentialFigure1} visualizes the process.
 Table \ref{yager_table1} shows the results of Yager's negation and Figure \ref{yager_figure} visualizes it.
 \begin{table}[H]
\centering
\caption{ $P_1$ exponential negative iteration results 1}
\begin{tabular}{ccccccc}
\midrule
\diagbox{$P_1$}{$n$} & 0     & 1     & 2     & 3     & 4     & 5     \\
\hline
      $p_{11}$      & 0     & 0.731 & 0.386 & 0.557 & 0.472 & 0.514  \\
      $p_{12}$      & 1     & 0.269 & 0.614 & 0.443 & 0.528 & 0.486  \\
\bottomrule
\end{tabular}
\label{negation_table1}
\end{table}

\begin{table}[H]
\center
\caption{ $P_1$ exponential negative iteration results 2}
\begin{tabular}{cccccc}
\midrule
\diagbox{$P_1$}{$n$} & 6     & 7     & 8     & 9     & 10 \\
\hline
      $p_{11}$      & 0.493 & 0.504 & 0.498 & 0.501 & 0.500 \\
      $p_{12}$     & 0.507 & 0.496 & 0.502 & 0.499 & 0.500 \\
\bottomrule
\end{tabular}
\label{negation_table2}
\end{table}

    \begin{figure}[H]
      \centering
      \includegraphics[scale=0.5]{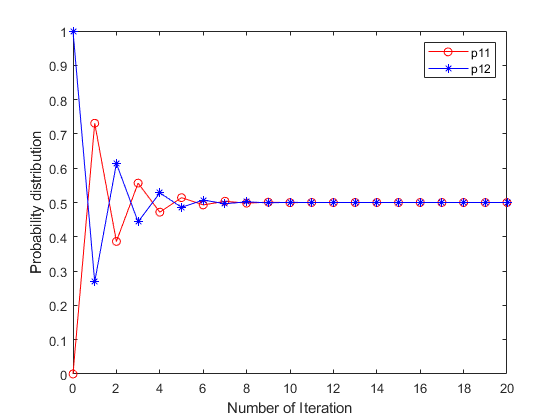}
      \caption{Change of probability distribution $P_1$ with the number of exponential negative iterations}
      \label{ExponentialFigure1}
    \end{figure}
      \par
\begin{table}[H]
\centering
\caption{$P_1$ Yager's negative iteration results \cite{yager2014maximum}}
\begin{tabular}{ccccccc}
\midrule
\diagbox{$P_1$}{$n$} & 0     & 1     & 2     & 3     & 4 &5\\
\hline
      $p_{11}$      & 0 & 1 & 0 & 1 & 0 &1\\
      $p_{12}$     & 1 & 0 & 1 & 0 & 1 &0\\
\bottomrule
\end{tabular}
\label{yager_table1}
\end{table}

\begin{figure}[H]
\centering
\includegraphics[scale=0.5]{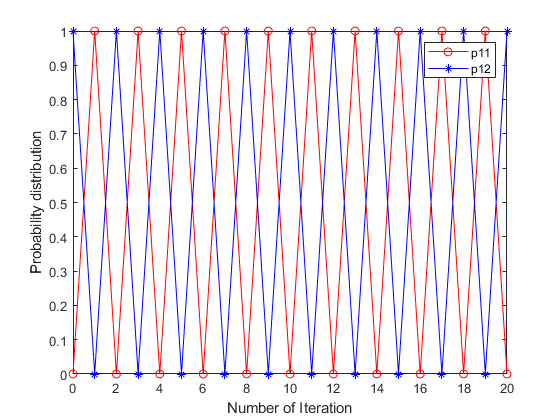}
\caption{Change of probability distribution $P_1$ with the number of Yager's negative iterations \cite{yager2014maximum}}
\end{figure}
\label{yager_figure}
 \end{example}

\par
For Yager's negation, a special result appears in this special case because $\overline{p_1}=1-p_2$ and $\overline{p_2}=1-p_1$, no matter how many iterations, the distribution will always alternate between these two distributions which do not reach the maximum entropy.
In this distribution, Yager's negation will no longer cause entropy increase.
However, Figure \ref{ExponentialFigure1} shows that the distribution of the proposed exponential negation will finally converge to the uniform distribution. The proposed exponential negation is an entropy increasing operation in this special case.
 \begin{example}
 \label{example4.2}
 Give a probability distribution as $P_2=\{p_{21},p_{22},p_{23}\}=\{0.1,0.4,0.5\}$
 Table \ref{negation_table3} show the results after given some iterative exponential negations on the distribution. Figure \ref{ExponentialFigure3} visualizes the process.
 Table \ref{table5} shows the results of Yager's negation and Figure \ref{yager_figure3} visualizes it.
\begin{table}[H]
\centering
\caption{$P_2$ exponential negative iterations results}
\begin{tabular}{cccccccc}
\midrule
 \diagbox{$P_2$}{$n$} & 0 & 1 &2 &3 &4 &5 &6 \\
\hline
      $p_{21}$      & 0.100   &0.414&0.306&0.342&0.330&0.334&0.333\\
      $p_{22}$      & 0.400   &0.307&0.342&0.331&0.334&0.333&0.333\\
      $p_{23}$      & 0.500   &0.278&0.352&0.327&0.335&0.332&0.333\\
\bottomrule
\end{tabular}
\label{negation_table3}
\end{table}
\begin{figure}[H]
  \centering
  \includegraphics[scale=0.5]{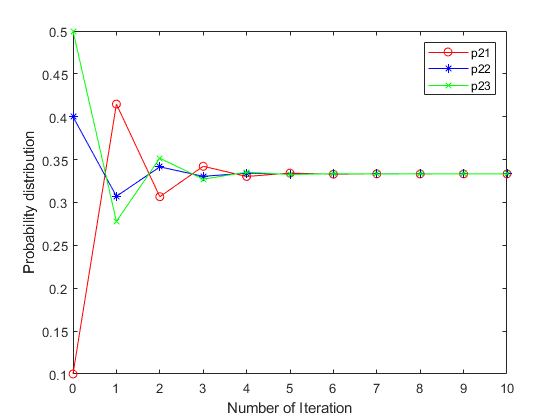}
  \caption{Change of probability distribution $P_2$ with the number of exponential negative iterations}
  \label{ExponentialFigure3}
\end{figure}
\begin{table}[H]
\centering
\caption{$P_2$ Yager's negative iterations results \cite{yager2014maximum}}
\begin{tabular}{cccccccc}
\midrule
 \diagbox{$P_2$}{$n$} & 0 & 1 &2 &3 &4 &5 &6 \\
\hline
    $p_{21}$    & 0.100   & 0.450  & 0.275 & 0.363 & 0.319 & 0.341 & 0.330 \\

    $p_{22}$    & 0.400   & 0.300   & 0.350  & 0.325 & 0.338 & 0.331 & 0.334 \\

    $p_{23}$     & 0.500   & 0.250  & 0.375 & 0.313 & 0.344 & 0.328 & 0.336 \\
\bottomrule
\end{tabular}
\label{table5}
\end{table}

\begin{figure}[H]
  \centering
  \includegraphics[scale=0.5]{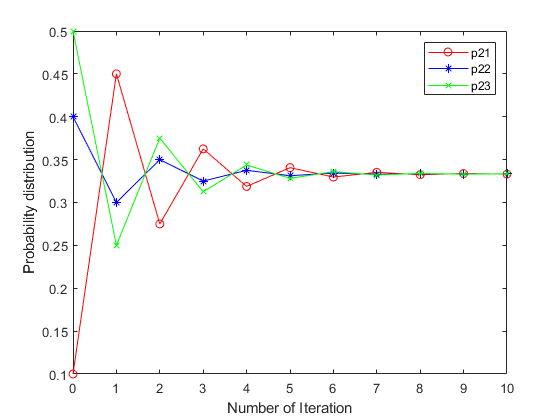}
  \caption{Change of probability distribution $P_2$ with the number of Yager's negative iterations \cite{yager2014maximum}}
  \end{figure}
  \label{yager_figure3}
\end{example}
 One advantage of the proposed negation is that it converges faster than Yager's negation is. In Example \ref{example4.2} , the Yager's method and the proposed negation can both converge to a uniform distribution. However, the convergence speed of Yager's negation is obviously slower than the convergence speed of our proposed negation. This is because the proposed exponential negation has a greater rate of change than Yager's negation, so it will be more efficient in obtaining information.
  When the precision reaches three decimal places, the tenth iteration of Yager's negation will converge to the uniform distribution,
  and our negation will converge to the uniform distribution of this precision after six iterations.

\begin{example}
 Give a probability distribution as $P_3=\{p_{31},p_{32},p_{33},p_{34},p_{35}\}=\{0.1,0.13,0.17,0.3,0.4\}$
 Table \ref{table6} show the results after given some iterative exponential negations on the distribution. Figure \ref{exponential5} visualizes the process.
 Table \ref{table7} shows the results of Yager's negation and Figure \ref{Yager5} visualizes it.
\begin{table}[H]
\centering
\caption{$P_3$ exponential negative iteration results}
\begin{tabular}{ccccccc}
\midrule
 \diagbox{$P_3$}{$n$} & 0 & 1 &2 &3 &4 &5 \\
\hline
    $p_{31}$     & 0.100   & 0.224 & 0.195 & 0.201 & 0.200 & 0.200    \\
       $p_{32}$     & 0.130  & 0.218 & 0.197 & 0.201 & 0.200 & 0.200    \\

       $p_{33}$     & 0.170  & 0.209 & 0.198 & 0.200 & 0.200 & 0.200    \\

       $p_{34}$     & 0.300   & 0.184 & 0.203 & 0.199 & 0.200 & 0.200    \\

       $p_{35}$     & 0.400   & 0.166 & 0.207 & 0.199 & 0.200 & 0.200  \\
\bottomrule
\end{tabular}
\label{table6}
\end{table}

  \begin{figure}[H]
    \centering
    \includegraphics[scale=0.5]{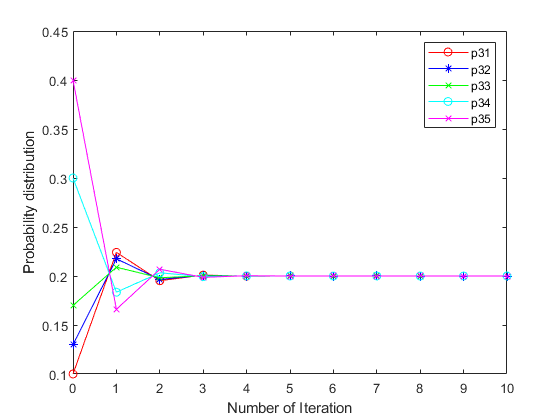}
    \caption{Change of probability distribution $P_3$ with the number of exponential negative iterations}
    \label{exponential5}
  \end{figure}

\begin{table}[H]
\centering
\caption{$P_3$ Yager's negative iteration results \cite{yager2014maximum}}
\begin{tabular}{ccccccc}
\midrule
 \diagbox{$P_3$}{$n$} & 0 & 1 &2 &3 &4 &5 \\
\hline
     $p_{31}$    & 0.100   & 0.225 & 0.194 & 0.202 & 0.200 & 0.200 \\

    $p_{32}$    & 0.130  & 0.218 & 0.196 & 0.201 & 0.200 & 0.200 \\

    $p_{33}$    & 0.170  & 0.208 & 0.198 & 0.200 & 0.200 & 0.200 \\

    $p_{34}$     & 0.300   & 0.175 & 0.206 & 0.198 & 0.200 & 0.200 \\

    $p_{35}$     & 0.400   & 0.150  & 0.213 & 0.197 & 0.201 & 0.200 \\
\bottomrule
\end{tabular}
\label{table7}
\end{table}

\begin{figure}[H]
  \centering
  \includegraphics[scale=0.5]{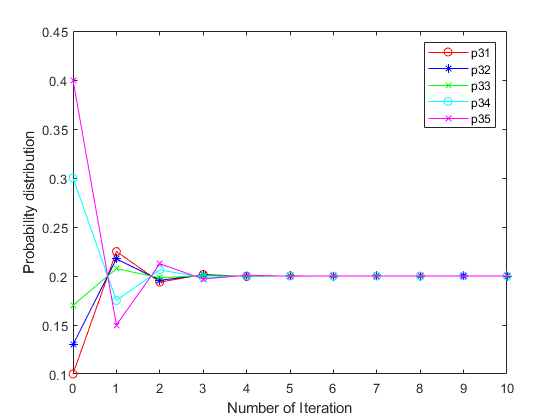}
  \caption{Change of probability distribution $P_3$ with the number of Yager's negative iterations \cite{yager2014maximum}}
  \label{Yager5}
\end{figure}
\end{example}
It can be seen from the Figure \ref{exponential5} and Figure \ref{Yager5} that the convergence images of the two negations are very similar which means the convergence process of the two is similar. It can be further found from the Table \ref{table6} and Table \ref{table7} that in this distribution,
in the case of three bits, they all converge to a uniform distribution in the fourth iteration.
But after comparing more accurate values, we found that Yager's negation converged to a uniform distribution at the 10th time, and the exponential negation converged to a uniform distribution at the 8th iteration.

\par
In general, the more elements in the distribution, the more similar the convergence process of the two negations, but for the special distribution $P_1$, the situation will become completely different.
The entropy in the proposed exponential negation will converge to a uniform distribution in this special binary case which is common in the real world. This property means that our proposed mathematical model can better reflect the actual situation.
\par
We can see that all the examples converge to the uniform distribution with our exponential negation.
In addition, an interesting point is that the numbers of iteration are different, the more elements the distribution has, the smaller number is required to converge to uniform distribution.
$P_1$ is required the biggest number in the three examples, we can infer that it also needs the biggest one in all distributions.
 As mentioned above, the entropy of $P_1$ is $H(P)=\sum_ip_{1i}ln(p_{1i})=0$, $H(P)>0$, $P_1$ has the minimum entropy in all distribution.
So it needs to do more negation calculations to reach the maximum entropy.

\section{CONCLUSION}
\label{five}
In this article, a new negation method, called as exponential negation, is presented. The proposed exponential negation has many desirable properties. For example, it is illustrated that all the probability distributions will converge to a uniform distribution after multiple negation iterations.
In addition, it can still converge very well even in the special binary situations. The proposed exponential negation can converge in a faster speed than the existed negation which means that it can get the opponent information from the probability distribution more efficient. Finally, it coincides with the second law of thermodynamics due to its entropy increase process.

\section*{Acknowledgment}
The work is partially supported by National Natural Science Foundation of China (Grant No. 61973332), JSPS Invitational Fellowships for Research in Japan (Short-term).
%
%
%
%
\section*{Conflict of Interest}
 The authors have no conflicts of interest to declare that are relevant to the content of this article.

\end{document}